\documentclass[twoside,11pt]{article}

\usepackage{jmlr2e}
\usepackage{wrapfig}
\usepackage{bm}

\usepackage{hyperref}

\usepackage{cleveref}

\usepackage{algorithm}
\usepackage{algorithmic}
\usepackage{amsmath}
\newtheorem{prop}{Proposition}

\newtheorem{innercustomthm}{Theorem}
\newenvironment{customthm}[1]
{\renewcommand\theinnercustomthm{#1}\innercustomthm}
{\endinnercustomthm}

\newcommand{\tr}{{\text{tr}}}

\newcommand{\dataset}{{\cal D}}
\newcommand{\dataMatrix}{\bm{X}}
\newcommand{\targetMatrix}{\bm{Y}}
\newcommand{\featureVector}{\bm{x}}
\newcommand{\targetVector}{\bm{y}}
\newcommand{\target}{y}
\newcommand{\featureVectorDimension}{D}
\newcommand{\targetVectorDimension}{Q}
\newcommand{\numHiddenNodes}{H} 
\newcommand{\numModules}{M}
\newcommand{\covarianceMatrix}{\Sigma}
\newcommand{\numExamples}{N}
\newcommand{\numEpochs}{T}
\newcommand{\encoder}{\bm{B}}
\newcommand{\decoder}{\bm{A}}
\newcommand{\modularAutoencoder}{\mathcal{W}}
\newcommand{\diversityParameter}{\lambda}
\newcommand{\modularLossFunction}{L}
\newcommand{\modularErrorFunction}{E}
\newcommand{\meanRegressionNetwork}{\overline{F}}
\newcommand{\identity}{\bm{I}}
\newcommand{\regressionNetwork}{F}
\newcommand{\modularRegressionNetwork}{\mathcal{F}}
\newcommand{\basisFunction}{\varphi}
\newcommand{\weightMatrix}{\bm{W}}
\newcommand{\basisFunctionParameters}{\rho}

\newcommand{\modL}{\left|\left|}
\newcommand{\modR}{\right|\right|}
\newcommand{\R}{\mathbb{R}}

\jmlrheading{1}{2015}{}{10/15}{12/15}{Henry Reeve and Gavin Brown}

\ShortHeadings{Modular Autoencoders for Ensemble Feature Extraction}{Reeve and Brown}
\firstpageno{1}

\begin{document}

\title{Modular Autoencoders for Ensemble Feature Extraction}

\author{\name Henry W J Reeve \email henrywjreeve@gmail.com \\
       \addr School of Computer Science \\
       The University of Manchester\\
       Manchester, UK
       \AND
       \name Gavin Brown \email gavin.brown@manchester.ac.uk \\
       \addr School of Computer Science \\
      The University of Manchester\\
      Manchester, UK}

\editor{Afshin Rostamizadeh}

\maketitle

\begin{abstract}
We introduce the concept of a Modular Autoencoder (MAE), capable of learning a set of diverse but complementary representations from unlabelled data, that can later be used for supervised tasks.  The learning of the representations is controlled by a trade off parameter, and we show on six benchmark datasets the optimum lies between two extremes: a set of smaller, independent autoencoders each with low capacity, versus a single monolithic encoding, outperforming an appropriate baseline. In the present paper we explore the special case of linear MAE, and derive an SVD-based algorithm which converges several orders of magnitude faster than gradient descent.
\end{abstract}


\section{Introduction}
In a wide variety of Machine Learning problems we wish to extract information from high dimensional data sets such as images or documents. Dealing with high dimensional data creates both computational and statistical challenges. One approach to overcoming these challenges is to extract a small set of highly informative features. These features may then be fed into a task dependent learning algorithm. In representation learning these features are learnt directly from the data \citep{bengio2013representation}. 

We consider a \textit{modular} approach to representation learning. Rather than extracting a single set of features, we extract multiple sets of features. Each of these sets of features is then fed into a separate learning module. These modules may then be trained independently, which addresses both computational challenges, by being easily distributable, and statistical challenges, since each module is tuned to just a small set of features. The outputs of the different classifiers are then combined, giving rise to a classifier ensemble.

Ensemble methods combine the outputs of a multiplicity of models in order to obtain an enriched hypothesis space whilst controlling variance \citep{friedman2001elements}. In this work we shall apply ensemble methods to representation learning in order to extract several subsets of features for an effective classifier ensemble. Successful ensemble learning results from a fruitful trade-off between accuracy and diversity within the ensemble. Diversity is typically encouraged, either through some form of randomisation, or by encouraging diversity through supervised training \citep{brown2005diversity}. 

We investigate an unsupervised approach to learning a set of diverse but complementary representations from unlabelled data. As such, we move away from the recent trend towards coupled dimensionality reduction in which the tasks of feature extraction and supervised learning are performed in unison \cite{gonen2014coupled, storcheus2015foundations}. Whilst coupled dimensionality reduction has been shown to improve accuracy for certain classification tasks \cite{gonen2014coupled}, the unsupervised approach allows us to use unlabelled data to learn a transferable representation which may be used on multiple tasks without the need for retraining \cite{bengio2013representation}.

We show that one can improve the performance of a classifier ensemble by first learning a diverse collection of modular feature extractors in a purely unsupervised way (see Section \ref{sec: empirical accuracy sec}) and then training a set of classifiers independently. Features are extracted using a Modular Autoencoder trained to simultaneously minimise reconstruction error and maximise diversity amongst reconstructions (see Section \ref{sec: NCAEIntro}). Though the MAE framework is entirely general to any activation function, in the present paper we focus on the linear case and provide an efficient learning algorithm that converges several orders of magnitude faster than gradient descent (see Section \ref{sec: linear algo}). The training scheme involves a hyper-parameter $\diversityParameter$. We provide an upper bound on $\lambda$, enabling a meaningful trade off between reconstruction error and diversity (see Section \ref{sec: diversity parameter upper bound}).
\begin{figure}\label{fig: mod autoencoder}
	\begin{center}
		\includegraphics[scale=0.3, trim = {0cm 0.2cm 0cm 0.2cm},clip]{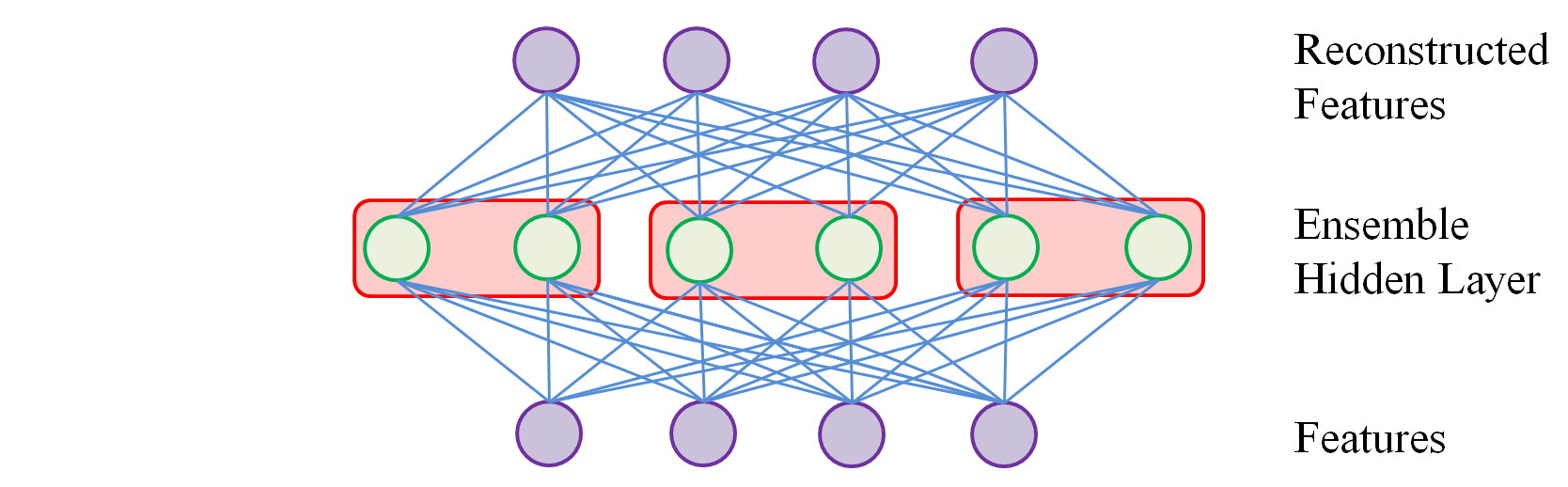}
		\caption{A Modular Autoencoder (MAE).}
	\end{center}
	\vspace{-0.5cm}
\end{figure}

\section{Modular Autoencoders}\label{sec: NCAEIntro}

A \textit{Modular Autoencoder} consists of an ensemble $\modularAutoencoder=\left\lbrace \left(\decoder_i,\encoder_i\right)\right\rbrace_{i=1}^{\numModules}$ consisting of $\numModules$ auto-encoder modules $\left(\decoder_i,\encoder_i\right)$, where each module consists of an encoder map $\encoder_i:\mathbb{R}^{\featureVectorDimension} \rightarrow \mathbb{R}^{\numHiddenNodes}$ from a $\featureVectorDimension$-dimensional feature space $\mathbb{R}^{\featureVectorDimension}$ to an $\numHiddenNodes$-dimensional representation space $\mathbb{R}^{\numHiddenNodes}$, and a decoder map $\decoder_i:\mathbb{R}^{\featureVectorDimension} \rightarrow \mathbb{R}^{\numHiddenNodes}$.  For reasons of brevity we focus on the linear case, where  $\decoder_i \in \mathbb{M}_{\featureVectorDimension \times \numHiddenNodes}\left(\mathbb{R}\right)$ and $\encoder_i\in \mathbb{M}_{\numHiddenNodes \times \featureVectorDimension}\left(\mathbb{R}\right)$ are matrices. See Figure \ref{fig: mod autoencoder}.

In order to train our Modular Autoencoders $\modularAutoencoder$ we introduce the following loss function
\begin{align}\label{NCAELoss}
\modularLossFunction_{\diversityParameter}\left(\modularAutoencoder,\featureVector\right):=&\frac{1}{\numModules}\sum_{i=1}^{\numModules}\overbrace{\modL \decoder_i\encoder_i\featureVector-\featureVector\modR^2}^{\text{reconstruction error}}-\diversityParameter \cdot \frac{1}{\numModules}\sum_{i=1}^{\numModules}\overbrace{\modL \decoder_i\encoder_i\featureVector-\frac{1}{\numModules}\sum_{j=1}^{\numModules}\decoder_j\encoder_j\featureVector\modR^2}
^{\text{diversity}},
\end{align}
for feature vectors $\featureVector\in \mathbb{R}^{\featureVectorDimension}$. The loss function $\modularLossFunction_{\diversityParameter}\left(\modularAutoencoder,\featureVector\right)$ is inspired by (but not identical to) the Negative Correlation Learning approach of by Liu and Yao for training supervised ensembles of neural networks \citep{liu1999ensemble}\footnote{See the Appendix for details.}. The first term corresponds to the squared reconstruction error typically minimised by Autoencoders \citep{bengio2013representation}. The second term encourages the reconstructions to be diverse, with a view to capturing different factors of variation within the training data. The hyper-parameter $\diversityParameter$, known as the diversity parameter, controls the degree of emphasis placed upon these two terms. We discuss its properties in Sections \ref{sec: two ext} and \ref{sec: diversity parameter upper bound}. Given a data set $\dataset \subset \mathbb{R}^{\featureVectorDimension}$ we train a Modular Autoencoder to minimise the error $\modularErrorFunction_{\diversityParameter}\left(\modularAutoencoder,\dataset\right)$, the loss function $\modularLossFunction_{\diversityParameter}\left(\modularAutoencoder,\featureVector\right)$ averaged across the data $\featureVector\in \dataset$.



\subsection{Between two extremes}\label{sec: two ext}

To understand the role of the diversity parameter $\diversityParameter$ we first look at the two extremes of $\diversityParameter = 0$ and $\diversityParameter =1$. If $\diversityParameter=0$ then no emphasis is placed upon diversity. Consequently $\modularLossFunction_{0}\left(\modularAutoencoder, \featureVector\right)$ is precisely the average squared error of the individual modules $\left(\decoder_i,\encoder_i\right)$. Since there is no interaction term, minimising $\modularLossFunction_{0}\left(\modularAutoencoder,\featureVector\right)$ over the training data is equivalent to training each of the auto-encoder modules independently, to minimise squared error. Hence, in the linear case $\modularErrorFunction_{0}\left(\modularAutoencoder,\dataset\right)$ is minimised by taking each $\encoder_i$ to be the projection onto the first $\numHiddenNodes$ principal components of the data covariance \citep{baldi1989neural}.

If $\lambda = 1$ then, by the Ambiguity Decomposition \citep{krogh1995neural},
\begin{align*}
\modularLossFunction_{1}\left(\modularAutoencoder,\featureVector\right)= \modL\frac{1}{\numModules}\sum_{i=1}^M\bm{\decoder}_i\bm{\encoder}_i\featureVector-\featureVector\modR^2.
\end{align*}
Hence, minimising $\modularLossFunction_{1}\left(\modularAutoencoder\right)$ is equivalent to minimising squared error for a single large Autoencoder $(\decoder,\encoder)$ with an $\numModules \cdot \numHiddenNodes$-dimensional hidden layer, where $\encoder = [\encoder_1^T, \cdots, \encoder_{\numModules}^T]^T$ and $\decoder={\numModules}^{-1}[\decoder_1,\cdots,\decoder_{\numModules}]$.

Consequently, moving $\diversityParameter$ between $0$ and $1$ corresponds to moving from training each of our autoencoder modules independently through to training the entire network as a single monolithic autoencoder. 

\subsection{Bounds on the diversity parameter}\label{sec: diversity parameter upper bound}

The diversity parameter $\lambda$ may be set by optimising the performance of a task-specific system using the extracted sets of features on a validation set. Theorem \ref{maeupperbound} shows that the search region may be restricted to the closed unit interval $\left[0,1\right]$.

\begin{theorem}\label{maeupperbound}
	Suppose we have a data set $\dataset$. The following dichotomy holds:
	\begin{itemize}
		\item If $\diversityParameter \leq 1$ then $\inf \modularErrorFunction_{\diversityParameter}\left(\modularAutoencoder,\dataset\right) \geq 0$.
		\item If $\lambda > 1$ then $\inf \modularErrorFunction_{\diversityParameter}\left(\modularAutoencoder,\dataset\right)  = - \infty$.
	\end{itemize}
	In both cases the infimums range over possible parametrisations for the ensemble $\modularAutoencoder$.
	
	Moreover, if the diversity parameter $\diversityParameter>1$ there exist ensembles $\modularAutoencoder$ with arbitrarily low error $\modularErrorFunction_{\diversityParameter}\left(\modularAutoencoder,\dataset\right)$ and arbitrarily high average reconstruction error.
\end{theorem}

Theorem \ref{maeupperbound} is a special case of Theorem \ref{ncupperbound1}, which is proved the appendix.

\section{An Efficient Algorithm for Training Linear Modular Autoencoders}\label{sec: linear algo}

One method to minimise the error $\modularErrorFunction_{\diversityParameter}\left(\modularAutoencoder,\dataset\right)$ would be to apply some form of gradient descent. However, Linear Modular Autoencoders we can make use of the Singular Value Decomposition to obtain a fast iterative algorithm for minimising the error $\modularErrorFunction_{\diversityParameter}\left(\modularAutoencoder,\dataset\right)$ (see Algorithm \ref{alg: linearNCAEalgo}).

\begin{algorithm}[h]
	\caption{Backfitting for Linear Modular Autoencoders \label{alg: linearNCAEalgo}}
	\begin{algorithmic}
		\STATE \textbf{Inputs:} $\featureVectorDimension \times \numExamples$ data matrix $\dataMatrix$, diversity parameter $\diversityParameter$, number of hidden nodes per module $\numHiddenNodes$, number of modules $\numModules$, maximal number of epochs T,
		\vspace{0.0cm}
		\STATE Randomly generate $\left\lbrace\left(\decoder_i,\encoder_i\right)\right\rbrace_{i=1}^{\numModules}$ and set $\covarianceMatrix \gets \dataMatrix \dataMatrix^T$,
		\FOR{t $ = 1 $ \TO T}
		\FOR{i $ = 1 $ \TO $M$}
		\STATE $\bm{Z}_i \gets {\numModules}^{-1}\sum_{j \neq i}\decoder_j\encoder_j$
		\STATE $\bm{\Phi} \gets  \left(\identity_{\featureVectorDimension}-\diversityParameter\cdot \bm{Z}_i\right)\covarianceMatrix\left(\identity_{\featureVectorDimension}-\diversityParameter\cdot \bm{Z}_i \right)^T$, where $\identity_{\featureVectorDimension}$ denotes the $\featureVectorDimension\times \featureVectorDimension$ identity matrix.
		\STATE ${\decoder_i} \gets [\bm{u}_1,\cdots,\bm{u}_{\numHiddenNodes}]$, where  $\left\lbrace \bm{u}_1,\cdots,\bm{u}_{\numHiddenNodes}\right\rbrace$ are the top eigenvectors of $\bm{\Phi}$.
		\STATE ${\encoder_i} \gets \left(1-\diversityParameter \cdot(\numModules-1)/\numModules\right)^{-1} \cdot {\decoder_i}^T\left(\identity_{\featureVectorDimension}-\diversityParameter\cdot \bm{Z}_i\right)$
		\ENDFOR
		\ENDFOR
		\RETURN Decoder-Encoder pairs $\left(\bm{A}_i,\bm{B}_i\right)_{i=1}^M$
	\end{algorithmic}
\end{algorithm}

Algorithm \ref{alg: linearNCAEalgo} is a simple greedy procedure reminiscent of the back-fitting algorithm for additive models \citep{friedman2001elements}. Each module is optimised in turn, leaving the parameters for the other modules fixed. The error $\modularErrorFunction_{\diversityParameter}\left(\modularAutoencoder,\dataset\right)$ decreases every epoch until a critical point is reached. 
	
\begin{theorem}\label{thm: algo prop}
	Suppose that $\covarianceMatrix=\dataMatrix \dataMatrix^T$ is of full rank. Let $\left(\modularAutoencoder_t\right)_{t=1}^{\numEpochs}$ be a sequence of parameters obtained by Algorithm \ref{alg: linearNCAEalgo}. For every epoch $t=\{1,\cdots,\numEpochs\}$, we have $E_{\lambda}\left(\modularAutoencoder_{t+1},\dataset\right)<\modularErrorFunction_{\diversityParameter}\left(\modularAutoencoder_t,\dataset\right)$, unless $\modularAutoencoder_{t}$ is a critical point for $\modularErrorFunction_{\diversityParameter}\left(\cdot, \dataset\right)$, in which case $\modularErrorFunction_{\diversityParameter}\left(\modularAutoencoder_{t+1},\dataset \right) \leq \modularErrorFunction_{\diversityParameter}\left(\modularAutoencoder_t, \dataset \right)$.
		\end{theorem}

Theorem \ref{thm: algo prop} justifies the procedure in Algorithm \ref{alg: linearNCAEalgo}. The proof is given in Appendix \ref{sec: deriv linear algo}. We compared Algorithm \ref{alg: linearNCAEalgo} with (batch) gradient descent on an artificial data set consisting of $1000$ data points randomly generated from a Gaussian mixture data set consisting of equally weighted spherical Gaussians with standard deviation $0.25$ and a mean drawn from a standard multivariate normal distribution. We measured the time for the cost to stop falling by at least $\epsilon = 10^{-5}$ per epoch for both Algorithm \ref{alg: linearNCAEalgo} and (batch) gradient descent. The procedure was repeated ten times. The two algorithms performed similarly in terms of minimum cost attained, with Algorithm \ref{alg: linearNCAEalgo} attaining slightly lower costs on average. However, as we can see from Table \ref{tbl: convergenceTimes}, Algorithm \ref{alg: linearNCAEalgo} converged several orders of magnitude faster than gradient descent.


\begin{table}[ht] 
	\centering
	\begin{tabular}{c c c c}
		\hline
		& Algorithm \ref{alg: linearNCAEalgo} & Gradient Descent & Speed up \\
		\hline 		
		\text{Minimum} & 0.1134 s & 455.2 s & 102.9$\times$ \\
		\text{Mean} & 1.4706 s & 672.9 s  & 1062.6$\times$ \\
		\text{Maximum} & 4.9842 s & 1871.5 s  & 6685.4$\times$ \\ 
		\hline
	\end{tabular}
		\caption{Convergence times for Algorithm \ref{alg: linearNCAEalgo} and batch gradient descent. \label{tbl: convergenceTimes}}
\end{table}

\section{Empirical results}\label{sec: empirical accuracy sec}

In this section we demonstrate the efficacy of Modular Autoencoders for extracting useful sets features for classification tasks. In particular, we demonstrate empirically that we can improve the performance of a classifier ensemble by first learning a diverse collection of modular feature extractors in an unsupervised way.

Our methodology is as follows. We take a training data set $\dataset=\left\lbrace (\featureVector_n,\target_n)\right\rbrace_{n=1}^N$ consisting of pairs of feature vectors $\featureVector_n$ and class labels $\target_n$. The data set $\dataset$ is pre-processed so each of the features have zero mean. We first train a Modular Autoencoder $\modularAutoencoder= \left(\decoder_i,\encoder_i\right)$. For each module $i$ we take $C_i$ to be the $1$-nearest neighbour classifier with the data set $\dataset_i = \left\lbrace (\encoder_i\featureVector_n,\target_n)\right\rbrace_{n=1}^{\numExamples}$. The combined prediction of the ensemble on a test point $\featureVector$ is defined by taking a modal average of the class predictions $\left\lbrace C_i(\encoder_i\featureVector)\right\rbrace_{i=1}^{\numModules}$.

We use a collection of six image data sets from \citet{larochelle2007empirical}, \textit{Basic, Rotations, Background Images} and \textit{Background Noise} variants of \textit{MNIST} as well as \textit{Rectangles} and \textit{Convex}. In each case we use a Modular Autoencoder consisting of ten modules ($\numModules=10$), each consisting of ten hidden nodes ($\numHiddenNodes=10$). The five-fold cross-validated test error is shown as a function of the diversity parameter $\diversityParameter$. We contrast with a natural baseline approach \textit{Bagging Autoencoders} (BAE) in which we proceed as described, but the modules $(\decoder_i,\encoder_i)$ are trained independently on bootstrapped samples from the data. In all cases, as the diversity parameter increases from zero the test error for features extracted using Modular Autoencoders falls well below the level attained by Bagging Autoencoders. As $\diversityParameter \rightarrow 1$ the ensemble error begins to rise, sometimes sharply.

\begin{figure}[h]
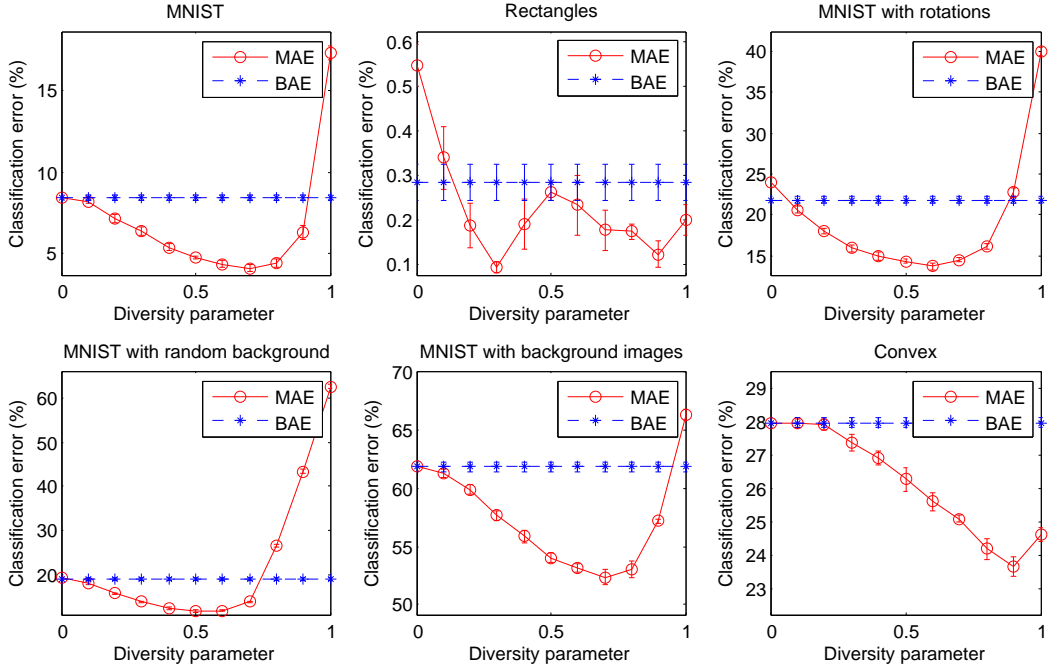

	\begin{center}
	\label{figure: ncaevs bagging mnist}
	\includegraphics[scale=0.7,trim = {0.1cm 6.4cm 0.1cm 7cm},clip]{{{"CombinedTestErrorPlots"}.pdf}}
	\vspace{-1.5cm}
	\caption{Test error for Modular Autoencoders (MAE) and Bagging Autoencoders (BAE).}
	\end{center}
\end{figure}

\section{Understanding Modular Autoencoders}\label{sec: understanding modular auto}

In this section we analyse the role of encouraging diversity in an unsupervised way with Modular Autoencoders and the impact this has upon supervised classification. 

\subsection{A more complex decision boundary}\label{subsec: 2Ddecisionboundary}

We begin by considering a simple two-dimensional example consisting of a Gaussian mixture with three clusters. In this setting we use a Linear Modular Autoencoder consisting of two modules, each with a single hidden node, so each of the feature extractors is simply a projection onto a line. We use a linear Softmax classifier on each of the extracted features. The probabilistic outputs of the individual classifiers are then combined by taking the mean average. The predicted label is defined to be the one with the highest probability.  Once again we observe the same trend as we saw in Section \ref{sec: empirical accuracy sec} - encouraging diversity leads to a substantial drop in the test error of our ensemble, with a test error of $21.3 \pm 1.3\%$ for $\lambda = 0$ and $12.8 \pm 1.0 \%$ for $\lambda = 0.5$.

To see why this is the case we contrast the features extracted when $\lambda = 0$ with those extracted when $\diversityParameter = 0.5$. Figure \ref{fig: gaussMixProjDensityDiversityZero} shows the projection of the class densities onto the two extracted features when $\lambda = 0$. No emphasis is placed upon diversity the two modules are trained independently to maximise reconstruction. Hence, the features extract identical information and there is no ensemble gain. Figure \ref{fig: 2DDecisionBoundaryNCAE} shows the resultant decision boundary; a simple linear decision boundary based upon a single one-dimensional classification. 

\begin{figure}[htb]
	\vspace{-1cm}
	\begin{center}
		\includegraphics[scale=0.6,trim = {1.5cm 10.75cm 1.5cm 9cm},clip]{{{"gaussianMixtureDiversityZeroDensities"}.pdf}}
		\vspace{-0.25cm}
		\caption{Projected class densities with $\diversityParameter=0$ \label{fig: gaussMixProjDensityDiversityZero}.
			}
	\end{center}
	
\end{figure}

\begin{figure}[htb]
	\vspace{-1cm}
	\begin{center}
		\includegraphics[scale=0.6,trim = {1.5cm 10.75cm 1.5cm 9cm},clip]{{{"gaussianMixtureDiversityHalfDensities"}.pdf}}
		\vspace{-0.25cm}
		\caption{Projected class densities with $\diversityParameter=0.5$ \label{fig: gaussMixProjDensityDiversityHalf}.
			}
	\end{center}	
\end{figure} 

\begin{figure}[htb]
	\begin{center}
		\includegraphics[scale=0.2,trim = {0 0 0 0},clip]{{{"2DDecisionBoundaryNCAE"}.png}}
		\vspace{-0.25cm}
		\caption{The decision boundary for $\diversityParameter = 0$ (left) and $\diversityParameter = 0.5$ (right)\label{fig: 2DDecisionBoundaryNCAE}.
		}
	\end{center}	
\end{figure} 

In contrast, when $\diversityParameter=0.5$ the two features yield diverse and complementary information. As we can see from Figure \ref{fig: gaussMixProjDensityDiversityHalf}, one feature separates class 1 from classes 2 and 3, and the other separates class 3 from classes 1 and 2. As we can see from the right of Figure \ref{fig: 2DDecisionBoundaryNCAE}, the resulting decision boundary accurately reflects the true class boundaries, despite being based upon two independently trained one-dimensional classifiers. This leads to the reduction in test error for $\lambda = 0.5$.

In general, Modular Autoencoders trained with the loss function defined in (\ref{NCAELoss}) extract diverse and complementary sets of features, whilst reflecting the main factors of variation within the data. Simple classifiers may be trained independently based upon these sets of features, so that the combined ensemble system gives rise to a complex decision boundary.

\subsection{Diversity of feature extractors}

In this Section we give further insight into the effect of diversity upon Modular Autoencoders. We return to the empirical framework of Section \ref{sec: empirical accuracy sec}. Figure \ref{fig: ensemble vs individual} plots two values test error for features extracted with Linear Modular Autoencoders. We plot both the average individual error of the classifiers (without ensembling the outputs) and the test error of the ensemble. In every case the average individual error rises as the diversity parameter moves away from zero. Nonetheless, the ensemble error falls as the diversity parameter increases (at least initially).

\begin{figure}[h]
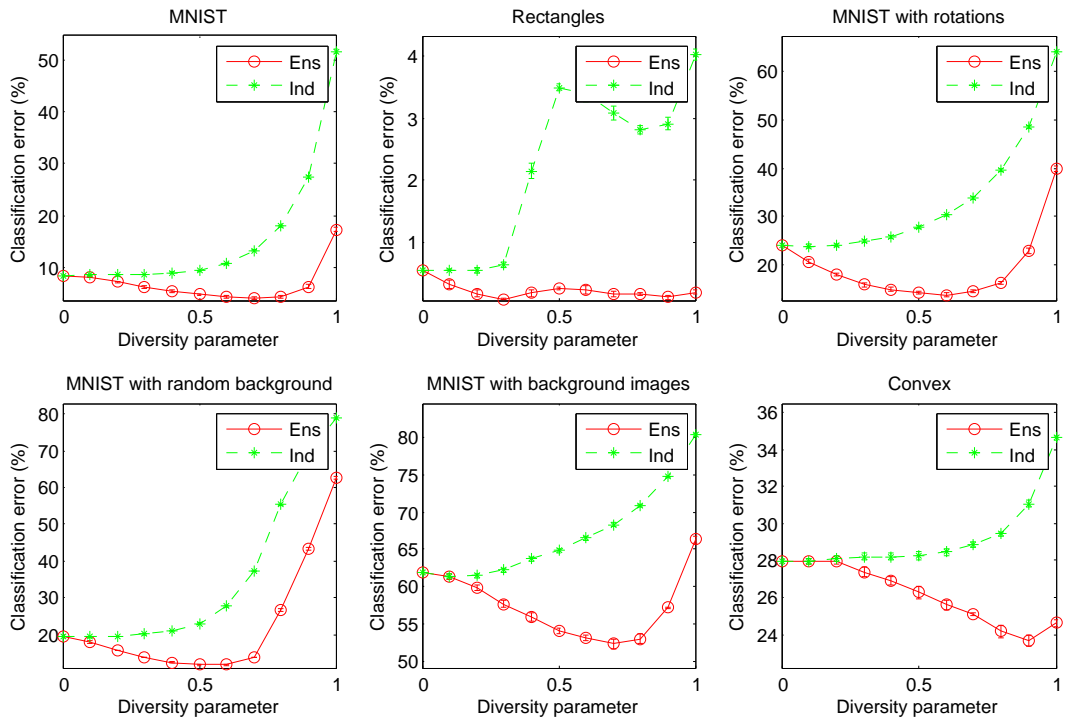

	\begin{center}
	\includegraphics[scale=0.7,trim = {0.1cm 6.4cm 0.1cm 7cm},clip]{{{"CombinedEnsembleVsIndividualErrorPlots"}.pdf}}
	\vspace{-1cm}
	\caption{Test error for the ensemble system (Ens) and the average individual error (Ind) \label{fig: ensemble vs individual}. Note that as the diversity parameter $\lambda$ increases, the individual modules sacrifice their own performance for the good of the overall set of modules - the average error rises, while the ensemble error falls.}
	\end{center}
\end{figure}

To see why the ensemble error falls whilst the average individual error rises we consider the metric structure of the different sets of extracted features. To compare the metric structure captured by different feature extractors, both with one another, and with the original feature space, we use the concept of \textit{distance correlation} introduced by  \citep{szekely2007measuring}. 

Given a feature extractor map $F$ (such as $\featureVector \mapsto \encoder_i \featureVector$) we compute $D\left(F,\mathcal{D}\right)$, the distance correlation based upon the pairs  $\left\lbrace \left(F(x),x\right): x \in \mathcal{D}\right\rbrace$. The quantity $D\left(F,\mathcal{D}\right)$ tells us how faithfully the extracted feature space for a feature map $F$ captures the metric structure of the original feature space. For each of our data sets we compute the average value of $D\left(F,\mathcal{D}\right)$ across the different feature extractors. To reduce computational cost we restrict ourselves to a thousand examples of both train and test data, $\mathcal{D}^{\text{red}}$. Figure \ref{fig: original vs extracted distance correlation} shows how the average value of $\numModules^{-1}\sum_{i=1}^MD\left(\encoder_i,\mathcal{D}^{\text{red}}\right)$ varies as a function of the diversity parameter. As we increase the diversity parameter $\diversityParameter$ we also reduce the emphasis on reconstruction accuracy. Hence, increasing $\diversityParameter$ reduces the degree to which the extracted features accurately reflect the metric structure of the original feature space. This explains the fall in individual classification accuracy we observed in Figure \ref{fig: ensemble vs individual}.

\begin{figure}[h]
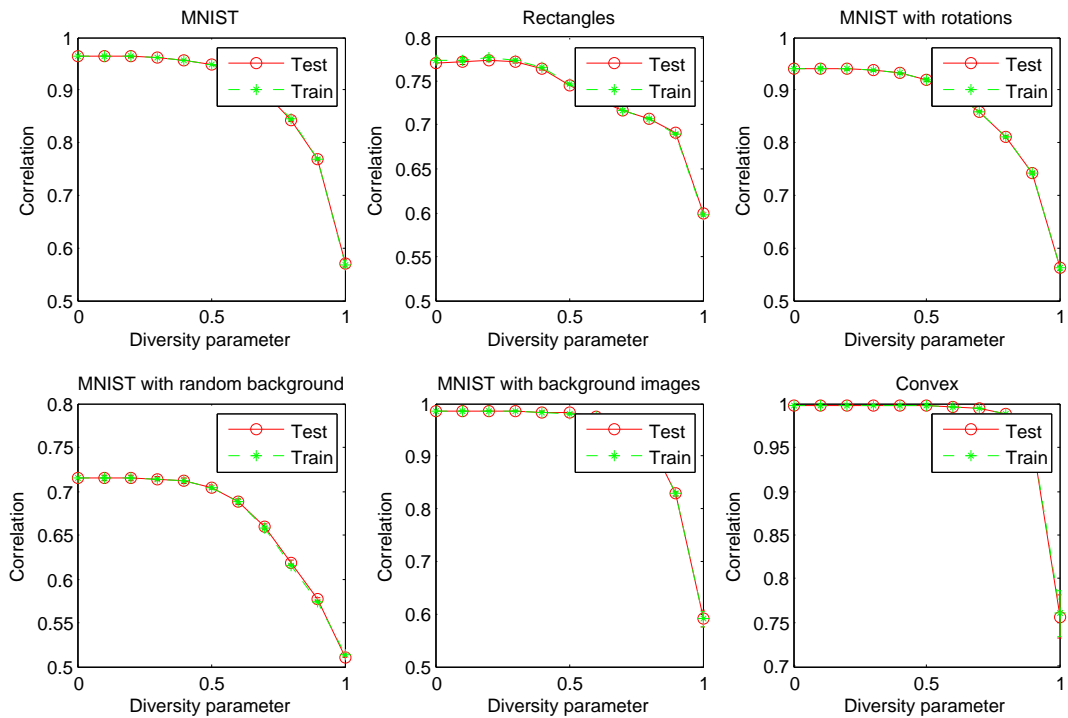

	\begin{center}
	\includegraphics[scale=0.7,trim = {0.1cm 6.4cm 0.1cm 7cm},clip]{{{"DistanceCorrelationOriginalSpaceExtractors"}.pdf}}
	\vspace{-1cm}
	\caption{Average distance correlation between extracted features.\label{fig: original vs extracted distance correlation}}
	\end{center}
\end{figure}

Given feature extractor maps $F$ and $G$ (such as $\featureVector \mapsto \encoder_i \featureVector$ and $\featureVector \mapsto \encoder_j \featureVector$), on a data set $\mathcal{D}$ we compute $C\left(F,G,\mathcal{D}\right)$, the distance correlation based upon the pairs $\left\lbrace \left(F(\bm{x}),G(\bm{x})\right): \bm{x} \in \mathcal{D}\right\rbrace$. The quantity $C\left(F,G,\mathcal{D}\right)$ gives us a measure of the correlation between the metric structures induced by $F$ and $G$. Again, to reduce computational cost we restrict ourselves to a thousand examples of both train and test data, $\mathcal{D}^{\text{red}}$. To measure the degree of diversity between our different sets of extracted features we compute the average pairwise correlation $C\left(\bm{B}_i,\bm{B}_j,\mathcal{D}^{\text{red}} \right)$, averaged across all pairs of distinct feature maps $\bm{B}_i,\bm{B}_j$ with $i \neq j$. Again we restrict ourselves to a thousand out-of-sample examples. Figure \ref{fig: pairwise extractors distance correlation} shows how the degree of metric correlation between the different sets of extracted features falls as we increase the diversity parameter $\diversityParameter$. Increasing $\diversityParameter$ places an increasing level of emphasis on a diversity of reconstructions. This diversity results in the different classifiers making different errors from one another enabling the improved ensemble performance we observed in Section \ref{sec: empirical accuracy sec}.

\begin{figure}[h]
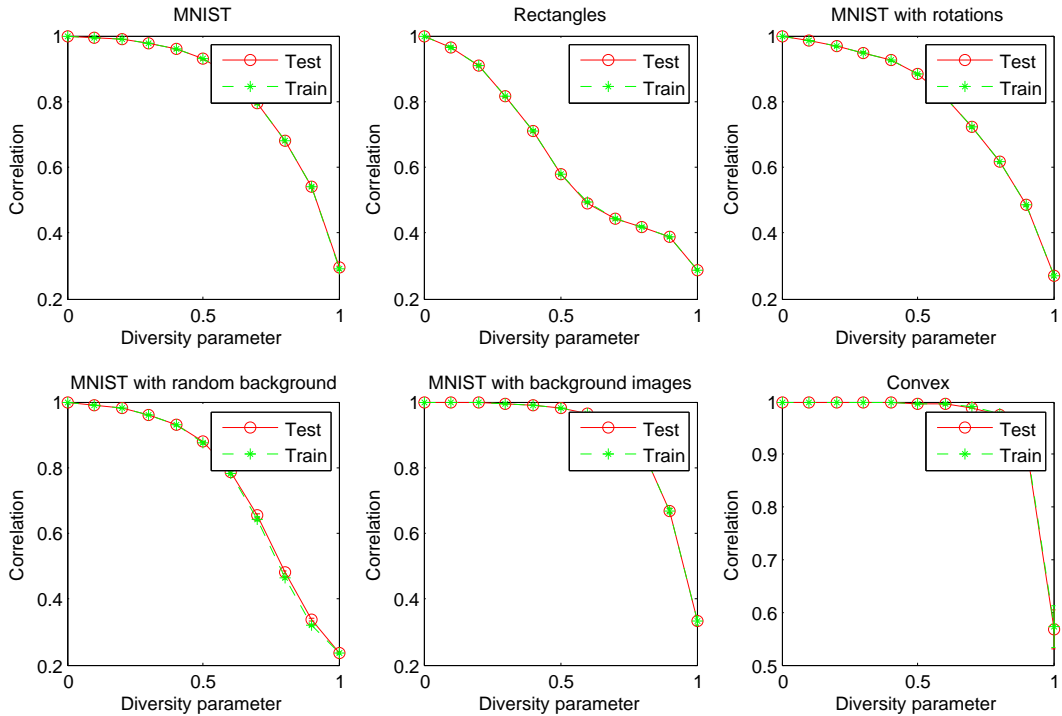

	\begin{center}
		\includegraphics[scale=0.7,trim = {0.1cm 6.4cm 0.1cm 7cm},clip]{{{"DistanceCorrelationDifferentExtractors"}.pdf}}
	\vspace{-1cm}
	\caption{Average pairwise distance correlation between different feature extractors\label{fig: pairwise extractors distance correlation}.}
		\end{center}
\end{figure}

\section{Discussion}

We have introduced a modular approach to representation learning where an ensemble of auto-encoder modules is learnt so as to achieve a diversity of reconstructions, as well as maintaining low reconstruction error for each individual module. We demonstrated empirically, using six benchmark data sets, that we can improve the performance of a classifier ensemble by first learning a diverse collection of modular feature extractors in an unsupervised way. We explored Linear Modular Autoencoders and derived an SVD-based algorithm which converges three orders of magnitude faster than gradient descent. In forthcoming work we extend this concept beyond the realm of auto-encoders and into a broader framework of modular manifold learning.

\acks{The research leading to these results has received funding from EPSRC Centre for Doctoral Training grant EP/I028099/1, the EPSRC Anyscale project EP/L000725/1 and from the AXLE project funded by the European Union's Seventh Framework Programme (FP7/2007-2013) under grant agreement no 318633. We would also like to thank Kostas Sechidis, Nikolaos Nikolaou, Sarah Nogueira and Charlie Reynolds for their useful comments, and the anonymous referee for suggesting several useful references.}

\newpage

\appendix

\section{Modular Regression Networks}\label{sec: mod reg net appendix}

We shall consider the more general framework of \textit{Modular Regression Networks} (MRN) which encompasses Modular Autoencoders (MAE). 

A Modular Regression Network $\modularRegressionNetwork = \left\lbrace \regressionNetwork_i\right\rbrace_{i=1}^{\numModules}$ is an ensemble system consisting of $\numModules$ mappings $\regressionNetwork_i:\R^{\featureVectorDimension}\rightarrow \R^{\targetVectorDimension}$. The MRN $\modularRegressionNetwork$ is trained using the following loss,
\begin{align}\label{eq: mrn loss}
\modularLossFunction_{\diversityParameter}\left(\modularRegressionNetwork,\featureVector,\targetVector\right): = \frac{1}{M}\sum_{i=1}^{\numModules}\overbrace{\modL\regressionNetwork_i\left(\featureVector\right)-\targetVector\modR^2}^{\text{error}}-\diversityParameter\cdot\frac{1}{M}\sum_{i=1}^M\overbrace{\modL \regressionNetwork_i(\featureVector)-\meanRegressionNetwork\left(\featureVector\right)\modR^2}^{\text{diversity}},
\end{align}
where $\featureVector$ is a feature vector, $\targetVector$ a corresponding output, and $\meanRegressionNetwork$ denotes the arithmetic average $\meanRegressionNetwork:=\frac{1}{\numModules}\sum_{i=1}^{\numModules}\regressionNetwork_i$. Given a data set $\dataset= \left\lbrace \left(\featureVector_n,\targetVector_n\right)\right\rbrace_{n=1}^{\numExamples}$ we let $\modularErrorFunction_{\diversityParameter}\left(\modularRegressionNetwork,\dataset\right)$ denote the loss $\modularLossFunction_{\diversityParameter}\left(\modularRegressionNetwork,\featureVector,\targetVector\right)$ averaged over $\left(\featureVector,\targetVector\right) \in \dataset$. The MRN $\modularRegressionNetwork$ is trained to minimise $\modularErrorFunction_{\diversityParameter}\left(\modularRegressionNetwork,\dataset\right)$.

\subsection{Investigating the loss function}

\begin{prop} \label{RewritingTheLossFunction}
	Given $\diversityParameter \in [0,\infty)$ and an MRN $\modularRegressionNetwork$, for each example $\left(\featureVector,\targetVector\right)$ we have
	\begin{align*}
	\modularLossFunction_{\diversityParameter}\left(\modularRegressionNetwork,\featureVector,\targetVector\right)= & \left(1-\diversityParameter\right)\cdot \frac{1}{\numModules}\sum_i\modL \regressionNetwork_i\left(\featureVector\right)-\targetVector\modR^2+\lambda \cdot \modL\meanRegressionNetwork\left(\featureVector\right)-\targetVector\modR^2.
	\end{align*}
\end{prop}
\begin{proof}
The result may be deduced from the Ambiguity Decomposition \citep{krogh1995neural}. 
\end{proof}

The following proposition relates MRNs to Negative Correlation Learning \citep{liu1999ensemble}.
\begin{prop}\label{mrn and ncl} Given an MRN $\modularRegressionNetwork$ and $\left(\featureVector,\targetVector\right) \in \dataset$ we have
	\begin{align*}
	\frac{\partial \modularLossFunction_{\diversityParameter}\left(\modularRegressionNetwork,\featureVector,\targetVector\right)}{\partial \regressionNetwork_i} = \frac{2}{M} \cdot 
	\left(\left(\regressionNetwork_i\left(\featureVector\right)-\targetVector\right)
	-\diversityParameter \cdot \left(\regressionNetwork_i\left(\featureVector\right)-\meanRegressionNetwork\left(\featureVector\right)\right)\right).
	\end{align*} 
\end{prop}
\begin{proof}
This follows from the definitions of $\modularLossFunction_{\diversityParameter}\left(\modularRegressionNetwork,\featureVector,\targetVector\right)$ and $\meanRegressionNetwork\left(\featureVector\right)$.
\end{proof}
In Negative Correlation Learning each network $\regressionNetwork_i$ is updated in parallel with rule 
	\begin{align*}
\theta^i\gets \theta^i - \alpha \cdot \frac{\partial \regressionNetwork_i}{\partial \theta^i} \left(\left(\regressionNetwork_i\left(\featureVector\right)-\targetVector\right)
	-\diversityParameter \cdot \left(\regressionNetwork_i\left(\featureVector\right)-\meanRegressionNetwork\left(\featureVector\right)\right)\right),
	\end{align*} 
for each example $\left(\featureVector,\targetVector\right) \in \dataset$ in turn, where $\theta^i$ denotes the parameters of $\regressionNetwork_i$ and $\alpha$ denotes the learning rate \citep[Equation 4]{liu1999ensemble}. By Proposition \ref{mrn and ncl} this is equivalent to training a MRN $\modularRegressionNetwork$ to minimise $\modularErrorFunction_{\diversityParameter}\left(\modularRegressionNetwork,\dataset\right)$ with stochastic gradient descent, using the learning rate $\numModules/2\cdot \alpha$.

\subsection{An upper bound on the diversity parameter}

We now focus on a particular class of networks. Suppose that there exists a vector-valued function $\basisFunction\left(\featureVector;\basisFunctionParameters\right)$, parametrised by $\basisFunctionParameters$. We assume that $\varphi$ is sufficiently expressive that for each possible feature vector $\featureVector \in \R^{\featureVectorDimension}$ there exists a choice of parameters $\basisFunctionParameters$ with $\basisFunction\left(\featureVector;\basisFunctionParameters\right) \neq \bm{0}$. Suppose that for each $i$ there exists a weight matrix $\weightMatrix^i$ and parameter vector $\basisFunctionParameters^i$ such that $\regressionNetwork_i\left(\featureVector\right) = \weightMatrix^i \basisFunction\left(\featureVector;\basisFunctionParameters^i\right)$. We refer to such networks as \textit{Modular Linear Top Layer Networks} (MLT). This is the natural choice in the context of regression and includes Modular Autoencoders with linear outputs.

\begin{theorem}\label{ncupperbound1}
Suppose we have a MLT $\modularRegressionNetwork$ and a dataset $\mathcal{D}$. The following dichotomy holds:
\begin{itemize}
\item If $\lambda \leq 1$ then $\inf \modularErrorFunction_{\diversityParameter}\left(\modularRegressionNetwork,\dataset\right) \geq 0$.
\item If $\lambda > 1$ then $\inf \modularErrorFunction_{\diversityParameter}\left(\modularRegressionNetwork,\dataset\right) = - \infty$.
\end{itemize}
In both cases the infimums range over possible parametrisations for the MRN $\modularRegressionNetwork$.

Moreover, if $\diversityParameter>1$ there exists parametrisations of $\modularRegressionNetwork$ with arbitrarily low error $\modularErrorFunction_{\diversityParameter}\left(\modularRegressionNetwork,\dataset\right)$ and arbitrarily high squared loss for the ensemble output $\meanRegressionNetwork$ and average squared loss for the individual regression networks $\regressionNetwork$.
\end{theorem}
\begin{proof}
It follows from Proposition \ref{RewritingTheLossFunction} that whenever $\diversityParameter \leq 1$, $\modularLossFunction_{\diversityParameter}\left(\modularRegressionNetwork,\featureVector,\targetVector\right) \geq 0$ for all choices of $\modularRegressionNetwork$ and all $\left(\featureVector,\targetVector\right) \in \dataset$. This implies the consequent in the case where $\diversityParameter\leq 1$.

We now address the implications of $\diversityParameter>1$. 

Take $\left(\tilde{\featureVector}, \tilde{\targetVector}\right) \in \dataset$. By the (MLT) assumption we may find parameters $\basisFunctionParameters$ so that $\basisFunction\left(\tilde{\featureVector};\basisFunctionParameters\right) \neq \bm{0}$. Without loss of generality we may assume that $0 \neq c =\basisFunction_1\left(\tilde{\featureVector};\basisFunctionParameters\right)$, where $\basisFunction_1\left(\featureVector;\basisFunctionParameters\right)$ denotes the first coordinate of $\basisFunction\left(\featureVector;\basisFunctionParameters\right)$. We shall leave $\basisFunctionParameters$ fixed and obtain a sequence $\left(\modularRegressionNetwork_q\right)_{q \in \mathbb{N}}$, where for each $q$ we have $\regressionNetwork_i^q\left(\featureVector\right) = \weightMatrix^{\left(i,q\right)} \basisFunction\left(\featureVector;\basisFunctionParameters\right)$ by choosing  $\weightMatrix^{\left(i,q\right)}$.

First take $\weightMatrix^{\left(i,q\right)}=0$ for all $i=3,\cdots,\numModules$, so 
\begin{align*}
\meanRegressionNetwork\left(\featureVector\right) = \frac{1}{\numModules} \left(\regressionNetwork_1(\featureVector)+\regressionNetwork_2(\featureVector)\right).
\end{align*}
In addition we choose $W^{\left(1,q\right)}_{kl}=W^{\left(2,q\right)}_{kl}=0$ for all $k>1$ or $l>1$. Finally we take $W^{\left(1,q\right)}_{11}=c\cdot \left(q^2+q\right)$ and $W^{\left(2,q\right)}_{11}=-c\cdot q^2$. It follows that for each $\left(\featureVector,\targetVector\right) \in \dataset$ we have,
\begin{align*}
\regressionNetwork_1(\featureVector) &= \weightMatrix^{\left(1,q\right)}\basisFunction\left(\featureVector,\basisFunctionParameters\right) = \left(c (q^2+q)\basisFunction_1\left(\featureVector;\basisFunctionParameters\right),0,\cdots,0\right)\\
\regressionNetwork_1(\featureVector) &= \weightMatrix^{\left(2,q\right)}\basisFunction\left(\featureVector,\basisFunctionParameters\right) = \left( -cq^2 \basisFunction_1\left(\featureVector;\basisFunctionParameters\right),0,\cdots,0\right),\\
\meanRegressionNetwork(\featureVector) &= \left(\numModules^{-1}cq\basisFunction_1\left(\featureVector;\basisFunctionParameters\right),0,\cdots,0\right).
\end{align*}
Noting that $\basisFunction_1\left(\featureVector;\basisFunctionParameters\right)=c \neq 0$ and $\left(\tilde{\featureVector}, \tilde{\targetVector}\right) \in \dataset$ we see that,
\begin{align}\label{ensembleasymptote}
\frac{1}{\numExamples}\sum_{n=1}^{\numExamples} \left(\meanRegressionNetwork(\featureVector_n)-\targetVector_n\right)^2 = \Omega(q^2). 
\end{align}
On the other hand we have,
\begin{align*}
\frac{1}{\numExamples}\sum_{n=1}^{\numExamples} \left(\regressionNetwork_1(\featureVector_n)-\targetVector_n\right)^2 &= \Omega(q^4),\\ 
\end{align*}
and clearly for all $i$,
\begin{align*}
\frac{1}{\numExamples}\sum_{n=1}^{\numExamples} \left(\regressionNetwork_1(\featureVector_n)-\targetVector_n\right)^2 &>0.\\ 
\end{align*}
Hence,
\begin{align*}
\frac{1}{\numExamples}\sum_{n=1}^{\numExamples}\frac{1}{\numModules}\sum_{i=1}^{\numModules} \left(\regressionNetwork_1(\featureVector_n)-\targetVector_n\right)^2 =\Omega(q^4). 
\end{align*}
Combining with Equation (\ref{ensembleasymptote}) this gives 
\begin{align*}
\modularErrorFunction_{\diversityParameter}\left(\modularRegressionNetwork^q,\dataset\right) = (1-\lambda) \cdot \Omega(q^4)+\lambda \cdot \Omega(q^2).
\end{align*}
Since $\lambda>1$ this implies
\begin{align}\label{combinederrorasymptote}
\modularErrorFunction_{\diversityParameter}\left(\modularRegressionNetwork^q,\dataset\right) = -\Omega(q^4).
\end{align}
By Equations \ref{ensembleasymptote} and \ref{combinederrorasymptote} we see that for any $Q_1,Q_2>1$, by choosing $q$ sufficiently large we have
\begin{align*}
\modularErrorFunction_{\diversityParameter}\left(\modularRegressionNetwork^q,\dataset\right)<-Q_1,
\end{align*}
and
\begin{align*}
\frac{1}{\numExamples}\sum_{n=1}^{\numExamples}\frac{1}{\numModules}\sum_{i=1}^{\numModules} \left(\regressionNetwork_i(\featureVector_n)-\targetVector_n\right)^2 \geq
\frac{1}{\numExamples}\sum_{n=1}^{\numExamples} \left(\meanRegressionNetwork(\featureVector_n)-\targetVector_n\right)^2 >Q_2,
\end{align*}
This proves the second item.
\end{proof}
The impact of Theorem \ref{ncupperbound1} is that whenever $\lambda>1$, minimising $E_{\lambda}$ will result in one or more parameters diverging. Moreover, the resultant solutions may be arbitrarily bad in terms of training error, leading to very poor choices of parameters.

\begin{theorem}
Suppose we have a MLT $\modularRegressionNetwork$ on a data set $\dataset$. Suppose we choose $i \in \{1,\cdots,\numModules\}$ and fix $\regressionNetwork_j$ for all $j \neq i$. The following dichotomy holds:
\begin{itemize}
\item If $\diversityParameter < \frac{\numModules}{\numModules-1}$ then $\inf 
\modularErrorFunction_{\diversityParameter}\left(\modularRegressionNetwork,\dataset\right) > -\infty$.
\item If $\diversityParameter > \frac{\numModules}{\numModules-1}$ then $\inf \modularErrorFunction_{\diversityParameter}\left(\modularRegressionNetwork,\dataset\right)= - \infty$.
\end{itemize}
In both cases the infimums range over possible parameterisations for the function $f_i$, with $f_j$ fixed for $j \neq i$.
\end{theorem}
\begin{proof}
We fix $\regressionNetwork_j$ for $j \neq i$. For each pair $\left(\featureVector,\targetVector\right)\in \dataset$ we consider $\modularLossFunction_{\diversityParameter}\left(\modularRegressionNetwork\right)$ for large $||\regressionNetwork_i(\featureVector)||$. By Proposition \ref{RewritingTheLossFunction}
we have
\begin{align*}
\modularLossFunction_{\diversityParameter}\left(\modularRegressionNetwork,\featureVector,\targetVector\right) = & \left(1-\lambda\right) \cdot \frac{1}{\numModules} \Omega\left(||\regressionNetwork_i(\featureVector)||^2\right) +\diversityParameter\cdot \Omega\left(||\frac{1}{\numModules}\regressionNetwork_i(\featureVector)||^2\right)\\
= & \left(1-\lambda \cdot \frac{\numModules-1}{\numModules}\right)\cdot \Omega\left(||\regressionNetwork_i(\featureVector)||^2\right).
\end{align*}
Hence, if $\diversityParameter < \frac{\numModules}{\numModules-1}$ we see that $\modularLossFunction_{\diversityParameter}\left(\modularRegressionNetwork,\featureVector,\targetVector\right)$ is bounded from below for each example $\left(\featureVector,\targetVector\right)$, for all choices of $\regressionNetwork_i$. This implies the first case.

In addition, the fact that $\basisFunction(\featureVector_1,\basisFunctionParameters) \neq 0$ for some choice of parameters $\basisFunctionParameters$ means that we may choose a sequence of parameters such that $||\regressionNetwork_i(\featureVector)|| \rightarrow \infty$ for one or more examples $\left(\featureVector,\targetVector\right) \in \dataset$. Hence, if $\diversityParameter > \frac{\numModules}{\numModules-1}$, we may choose weights so that $\modularLossFunction_{\diversityParameter}\left(\modularRegressionNetwork,\featureVector,\targetVector\right)\rightarrow -\infty$ for some examples $\left(\featureVector,\targetVector\right) \in \dataset$. The above asymptotic formula also implies that $\modularLossFunction_{\diversityParameter}\left(\modularRegressionNetwork,\featureVector,\targetVector\right)$ is uniformly bounded from above when $\diversityParameter > \frac{\numModules}{\numModules-1}$. Thus, we have $\inf \modularErrorFunction_{\diversityParameter}\left(\modularRegressionNetwork,\dataset\right) = - \infty$.
\end{proof}

\section{Derivation of the Linear Modular Autoencoder Training Algorithm}\label{sec: deriv linear algo}
In what follows we fix $\featureVectorDimension, \numExamples$, and $\numHiddenNodes<\featureVectorDimension$ and define
\[\mathcal{C}_{\featureVectorDimension,\numHiddenNodes}:= \left\lbrace \left(\decoder,\encoder\right): \decoder \in \R^{\featureVectorDimension\times \numHiddenNodes},\encoder \in \R^{\numHiddenNodes\times \featureVectorDimension}\right\rbrace.\]
We data set $\dataset \subset \R^{\featureVectorDimension}$, with $\featureVectorDimension$ features and $\numExamples$ examples, and let $\dataMatrix$ denote the $\featureVectorDimension \times \numExamples$ matrix given by $\dataMatrix=[\featureVector_1,\cdots, \featureVector_{\numExamples}]$. 

Given any $\diversityParameter \in [0,\infty)$ we define our error function by \begin{align*}
\modularErrorFunction_{\diversityParameter}\left(\modularAutoencoder,\dataset\right)= & \frac{1}{\numExamples}\sum_{n=1}^{\numExamples}\left(\frac{1}{\numModules}\sum_{j=1}^{\numModules} \left(\modL\featureVector_n-\decoder_j\encoder_j\featureVector_n\modR^2 -\diversityParameter \cdot \modL \decoder_j\encoder_j\featureVector_n -\frac{1}{\numModules}\sum_{k=1}^MA_kB_kx_n\modR^2\right)\right)\\
&= \frac{1}{\numExamples \cdot \numModules} \sum_{i=1}^{\numModules} \left(\modL\dataMatrix-\decoder_j\encoder_j\dataMatrix\modR^2 -\diversityParameter \cdot \modL \decoder_j\encoder_j\dataMatrix -\frac{1}{\numModules}\sum_{k=1}^MA_kB_k\dataMatrix\modR^2\right),
\end{align*}
where $\modularAutoencoder= \left(\left(\decoder_i,\encoder_i\right)\right)_{i=1}^{\numModules} \in \left(\mathcal{C}_{D,P}\right)^M$, and $\modL \cdot \modR$ denotes the Frobenius matrix norm.

\begin{prop}\label{FixAjBjforjneqiMinProp} Suppose we take $\dataMatrix$ so that $\covarianceMatrix = \dataMatrix\dataMatrix^T$ has full rank $\featureVectorDimension$ and choose $\diversityParameter <\numModules/(\numModules-1)$. We pick some $i \in \{1,\cdots,\numModules\}$, and fix $\decoder_j,\encoder_j$ for each $j \neq i$. Then we find $\left(\decoder_i,\encoder_i\right)$ which minimises $\modularErrorFunction_{\diversityParameter}\left(\modularAutoencoder,\dataset\right)$ by
\begin{enumerate}
\item Taking $\decoder_i$ to be the matrix whose columns consist of the $\featureVectorDimension$ unit eigenvectors with largest eigenvalues for the matrix
\begin{align*}
\left(\identity_{\featureVectorDimension}- \frac{\diversityParameter}{\numModules}\sum_{j \neq i}\decoder_j\encoder_j\right)\covarianceMatrix \left(\identity_{\featureVectorDimension}- \frac{\lambda}{\numModules}\sum_{j \neq i}\decoder_j\encoder_j\right)^T,
\end{align*}
\item Choosing $\encoder_i$ so that
\begin{align*}
\encoder_i=\left(1-\diversityParameter \cdot\frac{\numModules-1}{\numModules}\right)^{-1} \cdot \decoder_i^T\left(\identity_{\featureVectorDimension}- \frac{\diversityParameter}{\numModules}\sum_{j \neq i}\decoder_j\encoder_j\right).
\end{align*}
\end{enumerate}
Moreover, for any other decoder-encoder pair $\left(\tilde{\decoder}_i,\tilde{\encoder}_i\right)$ which also minimises $\modularErrorFunction_{\diversityParameter}\left(\modularAutoencoder,\dataset\right)$ (with the remaining pairs $\decoder_j, \encoder_j$ fixed) we have $\tilde{\decoder}_i\tilde{\encoder}_i =\decoder_i\encoder_i$.
\end{prop}

Proposition \ref{FixAjBjforjneqiMinProp} implies the following proposition from Section \ref{sec: linear algo}.

\begin{customthm}{\ref{thm: algo prop}}
	Suppose that $\covarianceMatrix$ is of full rank. Let $\left(\modularAutoencoder_t\right)_{t=1}^{\numEpochs}$ be a sequence of parameters obtained by Algorithm \ref{alg: linearNCAEalgo}. For every epoch $t=\{1,\cdots,\numEpochs\}$, we have $E_{\lambda}\left(\modularAutoencoder_{t+1},\dataset\right)<\modularErrorFunction_{\diversityParameter}\left(\modularAutoencoder_t,\dataset\right)$, unless $\modularAutoencoder_{t}$ is a critical point for $\modularErrorFunction_{\diversityParameter}\left(\cdot, \dataset\right)$, in which case $\modularErrorFunction_{\diversityParameter}\left(\modularAutoencoder_{t+1},\dataset \right) \leq \modularErrorFunction_{\diversityParameter}\left(\modularAutoencoder_t, \dataset \right)$.
\end{customthm}

\begin{proof}
 By Proposition \ref{FixAjBjforjneqiMinProp}, each update in Algorithm \ref{alg: linearNCAEalgo} modifies a decoder-encoder pair $(\decoder_i,\encoder_i)$ so as to minimise $\modularErrorFunction_{\diversityParameter}\left(\modularAutoencoder,\dataset \right)$, subject to the condition that $(\decoder_j,\encoder_j)$ remain fixed for $j \neq i$. Hence, $E_{\lambda}\left(\modularAutoencoder_{t+1},\dataset\right)\leq \modularErrorFunction_{\diversityParameter}\left(\modularAutoencoder_t,\dataset\right)$. 
 
 Now suppose $E_{\lambda}\left(\modularAutoencoder_{t+1},\dataset\right)= \modularErrorFunction_{\diversityParameter}\left(\modularAutoencoder_t,\dataset\right)$ for some $t$. Note that $E_{\lambda}\left(\modularAutoencoder_{t+1},\dataset\right)$ is a function of $\mathcal{C}=\left\lbrace \bm{C}_i \right\rbrace_{i=1}^{\numModules}$ where $\bm{C}_i = \bm{A}_i\bm{B}_i$ for $i=1,\cdots, \numModules$. We shall show that $\mathcal{C}_t$ is a critical point in terms for $\modularErrorFunction_{\diversityParameter}$. Since  $E_{\lambda}\left(\modularAutoencoder_{t+1},\dataset\right)= \modularErrorFunction_{\diversityParameter}\left(\modularAutoencoder_t,\dataset\right)$ we must have $\bm{C}_i^{t+1} = \bm{C}_i^t$ for $i =1,\cdots,M$. Indeed, Proposition \ref{FixAjBjforjneqiMinProp} implies that Algorithm \ref{alg: linearNCAEalgo} only modifies $\bm{C}_i$ when $E_{\lambda}\left(\modularAutoencoder,\dataset\right)$ is reduced (although the individual matrices $\decoder_i$ and $\encoder_i$ may be modified). Since $\bm{C}_i^{t+1} = \bm{C}_i^t$ we may infer that $\bm{C}_i^{t}$ attains the minimum value of $E_{\lambda}\left(\modularAutoencoder,\dataset\right)$ over the set of parameters such that $\bm{C}_j=\bm{C}_j^t$ for all $j \neq i$. Hence, at the point $\mathcal{C}_t$ we have ${\partial \modularErrorFunction_{\diversityParameter}}/{\partial \bm{C}_i}=0$ for each $i=1,\cdots, M$. Thus, ${\partial \modularErrorFunction_{\diversityParameter}}/{\partial \decoder_i}=0$ and ${\partial \modularErrorFunction_{\diversityParameter}}/{\partial \encoder_i}=0$, for each $i$, by the chain rule. 
\end{proof}

To prove Proposition \ref{FixAjBjforjneqiMinProp} we require two intermediary lemmas. The first is a theorem concerning  Rank Restricted Linear Regression.

\begin{theorem}\label{RankRestrictedLinearReg} Suppose we have $\featureVectorDimension\times \numExamples$ data matrices $\dataMatrix,\targetMatrix$. We define a function $\modularErrorFunction:\mathcal{C}_{\featureVectorDimension,\numHiddenNodes} \rightarrow \R$ by
\begin{align*}
\modularErrorFunction\left(\decoder,\encoder\right) = \modL \targetMatrix-\decoder \encoder \dataMatrix \modR^2.
\end{align*}
Suppose that the matrix $\dataMatrix \dataMatrix^T$ is invertible and define $\Sigma := (\targetMatrix \dataMatrix^T)(\dataMatrix \dataMatrix^T)^{-1}(\dataMatrix \targetMatrix^T)$. Let $\bm{U}$ denote the $\numExamples\times \featureVectorDimension$ matrix who's columns are the $\featureVectorDimension$ unit eigenvectors of $\Sigma$ with largest eigen-values. Then the minimum for $\modularErrorFunction$ is attained by taking,
\begin{align*}
\decoder&=\bm{U}\\
\encoder& = \bm{U}^T (\targetMatrix \dataMatrix^T)(\dataMatrix \dataMatrix^T)^{-1}.
\end{align*}
\end{theorem}
\begin{proof}
See \citet[Fact 4]{baldi1989neural}.
\end{proof}
Note that the minimal solution is not unique. Indeed if $\decoder,\encoder$ attain the minimum, then so does $\decoder\bm{C}$, $\bm{C}^{-1}\encoder$ for any invertible $\numHiddenNodes\times \numHiddenNodes$ matrix $\bm{C}$.

\begin{lemma}\label{CompleteSquareTypeLemma} Suppose we have $\featureVectorDimension\times \numExamples$ matrices $\dataMatrix$ and $\targetMatrix_1,\cdots, \targetMatrix_Q$, and scalars $\alpha_1,\cdots,\alpha_Q$ such that $\sum_{q=1}^Q\alpha_q>0$.   
Then we have
\begin{align*}
\arg & \min_{\left(\decoder,\encoder\right) \in \mathcal{C}_{\featureVectorDimension,\numHiddenNodes}} \left\lbrace \sum_{q=1}^Q\alpha_q ||\targetMatrix_q-\decoder\encoder\dataMatrix||^2\right\rbrace\\
&=\arg \min_{\left(\decoder,\encoder\right) \in \mathcal{C}_{\featureVectorDimension,\numHiddenNodes}} \left\lbrace ||\left(\sum_{q=1}^Q\tilde{\alpha}_q\targetMatrix\right)-\decoder\encoder\dataMatrix||^2 \right\rbrace,
\end{align*}
where $\tilde{\alpha}_q=\alpha_q/\left(\sum_{q'=1}^Q\alpha_{q'}\right)$ .
\end{lemma}
\begin{proof}
We use the fact that under the Frobenius matrix norm, $||\bm{M}||^2 = \tr(\bm{M}\bm{M}^T)$ for matrices $\bm{M}$, where $\tr$ denotes the trace operator. Note also that the trace operator is linear and invariant under matrix transpositions. Hence, we have
\begin{align*}
\sum_{q=1}^Q\alpha_q \modL\targetMatrix_q-\decoder\encoder\dataMatrix\modR^2
\end{align*}
\begin{align*}
=& \sum_{q=1}^Q\alpha_q \cdot \tr\left((\targetMatrix_q-\decoder\encoder\dataMatrix)(\targetMatrix_q-\decoder\encoder\dataMatrix)^T\right)\\
= & \sum_{q=1}^Q \alpha_q \cdot \tr\left(\targetMatrix_q\targetMatrix_q^T-2(\decoder\encoder)\dataMatrix\targetMatrix_q^T+(\decoder\encoder)\dataMatrix \dataMatrix(\decoder\encoder)^T\right)\\
= & \sum_{q=1}^Q\alpha_q \modL \targetMatrix_q\modR^2- \tr\left(2(\decoder\encoder)\dataMatrix\left(\sum_{q=1}^Q\alpha_q \targetMatrix_q\right)^T\right)+ \tr\left( \left(\sum_{q=1}^Q\alpha_q\right) (\decoder \encoder)\dataMatrix \dataMatrix^T(\decoder\encoder)^T\right).
\end{align*}
Note that we may add constant terms (ie. terms not depending on $\decoder$ or $\encoder$) and multiply by positive scalars without changing the minimising argument. Hence, dividing by $\sum_{q=1}^Q\alpha_q>0$ and adding a constant we see that the minimiser of the above expression is equal to the minimiser of 
\begin{align*}
\tr\left((\decoder\encoder)\dataMatrix \dataMatrix^T(\decoder\encoder)^T\right)+
\tr\left(2(AB)X\left(\sum_{q=1}^Q\tilde{\alpha}_q \targetMatrix_q\right)^T\right)
+ \tr\left(\left(\sum_{q=1}^Q\tilde{\alpha}_q \targetMatrix_q\right)\left(\sum_{q=1}^Q\tilde{\alpha}_q \targetMatrix_q\right)^T\right).
\end{align*}
Moreover, by the linearity of the trace operator this expression is equal to
\begin{align*}
\modL\left(\sum_{q=1}^Q\tilde{\alpha}_q \targetMatrix_q\right)-\decoder \encoder \dataMatrix\modR^2.
\end{align*}
This proves the lemma.
\end{proof}

\begin{proof}[Proposition \ref{FixAjBjforjneqiMinProp}]
We begin observing that if we fix $\decoder_j,\encoder_j$ for $j \neq i$, then minimising $\modularErrorFunction_{\diversityParameter}\left(\modularAutoencoder,\dataset\right)$ is equivalent to minimising 
\begin{align*}
\modL\dataMatrix-\decoder_i\encoder_i\dataMatrix\modR^2&-\diversityParameter \left(1-\frac{1}{\numModules}\right)^2\modL\frac{1}{\numModules-1}\cdot \bm{S}_{-i}-\decoder_i\encoder_i\dataMatrix\modR^2-\\
& \frac{\diversityParameter}{\numModules^2}\sum_{j \neq i}\modL\left(\numModules \decoder_j\encoder_j \dataMatrix-\bm{S}_{-i}\right)-
\decoder_i\encoder_i\dataMatrix \modR^2, 
\end{align*}
where $\bm{S}_{-i} = \sum_{j\neq i }\decoder_j\encoder_j\dataMatrix$. This holds as the above expression differs from $\modularErrorFunction_{\diversityParameter}\left(\modularAutoencoder, \dataset\right)$ only by a multiplicative factor of $NM$ and some constant terms which do not depend upon $\decoder_i,\encoder_i$. 

By Lemma \ref{CompleteSquareTypeLemma}, minimising the above expression in terms of $\decoder_i, \encoder_i$ is equivalent to minimising 
\begin{align} \label{expressionToMinimiseIniSimpleForm}
||\targetMatrix-\decoder_i\encoder_i\dataMatrix||,
\end{align}
with 
\begin{align*}
\targetMatrix  = & \left(1-\lambda\left(\left(1-\frac{1}{\numModules}\right)^2+\frac{\numModules-1}{\numModules^2}\right)\right)^{-1}\\
& \cdot\left(\dataMatrix-\diversityParameter \cdot \left(  
\left(1-\frac{1}{\numModules}\right)^2\frac{1}{\numModules-1} \cdot \bm{S}_{-i} +\frac{1}{\numModules^2}\sum_{j\neq i}\left(\numModules \decoder_j\encoder_j \dataMatrix-\bm{S}_{-i}\right)
\right)\right).
\end{align*}
Here we use the fact that $\diversityParameter <\numModules/(\numModules-1)$, so 
\begin{align*}
1&-\diversityParameter\left(\left(1-\frac{1}{\numModules}\right)^2+\frac{\numModules-1}{\numModules^2}\right)= 1-\diversityParameter \cdot \frac{\numModules-1}{\numModules}>0.
\end{align*}
We may simplify our expression for $\targetMatrix$ as follows,
\begin{align*}
\targetMatrix  = \left(1-\diversityParameter \cdot\frac{\numModules-1}{\numModules}\right)^{-1} \cdot \left(\identity_{\featureVectorDimension} - \frac{\diversityParameter}{\numModules}\sum_{j \neq i}\decoder_j\encoder_j\right)\dataMatrix.
\end{align*}
By Theorem \ref{RankRestrictedLinearReg}, we may minimise the expression in \ref{expressionToMinimiseIniSimpleForm} by taking $\decoder_i$ to be the matrix whose columns consist of the $\featureVectorDimension$ unit eigenvectors with largest eigenvalues for the matrix
\begin{align*}
\left(\identity_{\featureVectorDimension} - \frac{\diversityParameter}{\numModules}\sum_{j \neq i}\decoder_j\encoder_j\right)\left(\dataMatrix \dataMatrix^T\right)\left(\identity_{\featureVectorDimension} - \frac{\diversityParameter}{\numModules}\sum_{j \neq i}\decoder_j\encoder_j\right)^T,
\end{align*}
and setting
\begin{align*}
\encoder_i=\left(1-\diversityParameter \cdot\frac{\numModules-1}{\numModules}\right)^{-1} \cdot \decoder_i^T\left(\identity_{\featureVectorDimension} - \frac{\diversityParameter}{\numModules}\sum_{j \neq i}\decoder_j\encoder_j\right).
\end{align*}
This completes the proof of the proposition.
\end{proof}

\clearpage

\bibliography{modularAEbib} 

\end{document}